\documentclass{sig-alternate}

\usepackage{ifthen}
\usepackage{xcolor}
\usepackage[pdfborder={0 0 0}]{hyperref}
\newcommand{\markupdraft}[2]{% {#1: {color|display} command}{#2: desired color or text}
    \ifthenelse{\equal{#1}{display}}{#2}{}%                 % display only in draft version
    \ifthenelse{\equal{#1}{color}}{\color{#2}}{}%           % colored only in draft (for \new command)
}
\newcommand{\notecolored}[3][]{\markupdraft{display}{{\color{#2}\noindent[Note (#1): #3]}}}
\newcommand{\newcolored}[3][]{{\markupdraft{color}{#2}#3}%  % kept in the final print
    \ifthenelse{\equal{#1}{}}{}{\markupdraft{display}{{\color{yellow!70!black}[#1]}}}} 

\newcommand{\new}[2][]{\newcolored[#1]{blue}{#2}}%  % kept in the final print
\newcommand{\nnew}[2][]{\newcolored[#1]{red}{#2}}%  % kept in the final print
\newcommand{\note}[2][]{\notecolored[#1]{green}{#2}}    

\newcommand{\todo}[2][]{\markupdraft{display}{{\color{red}\noindent++TODO: #2 ++}}}
\renewcommand{\markupdraft}[2]{}  % remove all todo's, notes and coloring of changes
\newcommand{\yohe}[1]{\markupdraft{display}{{\color{magenta}[Note (Youhei): #1]}}}
\newcommand{\yann}[1]{\markupdraft{display}{{\color[rgb]{0.0,0.8,0.4}[NdYann: #1]}}}

\usepackage{cmap}
\usepackage[utf8]{inputenc}
\usepackage[T1]{fontenc}
\usepackage[english]{babel}
\usepackage{graphicx}
\usepackage{latexsym}
\usepackage{cite}

\usepackage{dsfont}  % mathds : This uses Type 1 fonts

\usepackage{amssymb}
\usepackage{amsmath}
\usepackage{amsfonts}
\usepackage{mathtools}    
\usepackage{stmaryrd}%%for llbracket

\newcommand{\indicator}[1]{\mathds{1}_{\left[ {#1} \right] }}
\newcommand\argmax{\mathop{\arg\,\max}\limits}%

\newcommand{\X}{X} % search space
\newcommand{\R}{\mathbb{R}} % real space
\newcommand{\rmd}{\mathrm{d}\hspace{-0.02em}}
\newcommand{\dx}{\rmd{}x} % dx
\newcommand{\RM}{\dx} % reference measure
\newcommand{\np}{\beta} % natural parameter
\newcommand{\ep}{\eta} % expectation parameter
\newcommand{\E}{\mathbb{E}} % expectation
\newcommand{\T}{\mathrm{T}} % transpose
\newcommand{\qm}{q_\theta^<} % lower quantile
\newcommand{\qp}{q_\theta^\leq} % upper quantile
\newcommand{\Wf}{W^f_\theta} % weighted preferene
\newcommand{\qmt}{q_{\theta^t}^<} % lower quantile at t
\newcommand{\Wft}{W^f_{\theta^t}} % weighted preferene at t
\newcommand{\Jtwo}[2]{J(#1\,|\,#2)} % expected fitness given theta^t
\newcommand{\Jt}[1]{J(#1\,|\,\theta^t)} % expected fitness given theta^t
\newcommand{\Jtt}[1]{J(#1\,|\,\theta^{t+1})} % expected fitness given theta^{t+1}
\newcommand{\FM}{\mathcal{I}} % Fisher metric
\newcommand{\tnabla}{\tilde \nabla} % natural gradient
\newcommand{\deltat}{\ensuremath{\delta\hspace{-.06em}t\hspace{0.05em}}}
\newcommand{\KL}[2]{D_\mathrm{KL}(#1 \!\parallel\! #2)} % KL divergence
\newcommand{\dtheta}{\delta \theta} % \delta \theta
\newcommand{\onehalf}{\frac{1}{2}} % 1/2
\newcommand{\norm}[1]{\lVert#1\rVert} % Euclidean norm
\newcommand{\rkm}{\mathrm{rk}^{<}} % lower rank
\newcommand{\rkp}{\mathrm{rk}^{\leq}} % upper rank
\newcommand{\wi}{\widehat w_i} % approximated weight
\newcommand{\barw}{\bar w} % averaged weight
\newcommand{\Gt}{G^t} % natural gradient estimate at t
\newcommand{\QPq}{Q_{P}^{q}} % q-quantile under P
\newcommand{\QPtq}{Q_{P_{\theta^t}}^{q}} % q-quantile under P_{\theta^t}
\newcommand{\QPttq}{Q_{P_{\theta^{t+\deltat}}}^{q}} % q-quantile under P_{\theta^{t+\deltat}}

\newcommand{\mmax}{m_\text{max}}
\newcommand{\QPthetaq}{Q_{P_\theta}^q}
\newcommand{\Ht}{H_t}
\newcommand{\thetatt}{\theta^{t+\deltat}}

\newcommand{\deq}{\mathrel{\mathop:}=}

\newtheorem{theorem}{Theorem}
\newtheorem{statement}[theorem]{Statement}
\newtheorem{lemma}[theorem]{Lemma}
\newtheorem{proposition}[theorem]{Proposition}
\newtheorem{corollary}[theorem]{Corollary}
\newdef{example}{Example}
\newdef{definition}{Definition}
\newdef{remark}{Remark}

\begin{document}

% Copyright
\conferenceinfo{FOGA'13,} {January 16–20, 2013, Adelaide, Australia.} 
\CopyrightYear{2013} 
\crdata{978-1-4503-1990-4/13/01} 
\clubpenalty=10000 
\widowpenalty = 10000

\title{Objective Improvement in\\
Information-Geometric Optimization}
%\subtitle{}

\numberofauthors{2} %  in this sample file, there are a *total*
\author{
% 1st. author
\alignauthor
Youhei Akimoto\\
       \affaddr{Project TAO -- INRIA Saclay}\\
       \affaddr{LRI, B\^at.~490, Univ.~Paris-Sud}\\
       \affaddr{91405 Orsay, France}\\
       \email{Youhei.Akimoto@lri.fr}
% 2nd. author
\alignauthor
Yann Ollivier\\
       \affaddr{CNRS \& Univ.~Paris-Sud}\\
       \affaddr{LRI, B\^at.~490}\\
       \affaddr{91405 Orsay, France}\\
       \email{yann.ollivier@lri.fr}
}
\date{14 Aug., 2012}

\maketitle

\yann{I guess the paper is ready then?}
\yohe{Yes. I think so. }

\note[Reviewer's comments]{
\begin{itemize}
 \item Section 3, first sentence: ``justify the framework'' -> theoretically nice property?
 \item The weakness of the paper, I think, is in laying out the purpose of the endeavor.
 \item I think the purpose of the IGO is unclear. Is it meant to propose
 a novel formalism for derivation of exiting algorithms? Or is it
 deriving new algorithms? I think the latter because even example 1,
 which is supposed to be PBIL, there appear to be differences, such as
 with the quartile weighted selection. Given this, it would be nice to
 know if this really works, but no experimental results are given. In any
 event, more discussion of whether this is or is not the case would be
 useful.
\yann{Nothing to be done here, I guess.} 
\yohe{I agree with you. We state in the second paragraph in section 1 
``the IGO framework not only provides information-theoretic derivations for
existing algorithms but automatically offers new algorithms''}
 
 This paper purports to show that IGO gives monotone improvement with an
 infinite population, which is an important result. It then states in a
 brief section at the end that this will hold in a finite population and
 even in a small population \yann{We explicitly say it will not hold with
 a small population}, so there must be some empirical results. Does that
 mean that you can derive a form of, say PBIL, which will have this
 property while the ordinary version does not? Again, emphasize this,
 clarify this, discuss this; it would make the paper stronger. \yann{I
 don't really know what to do here... Of course using PBIL with a very
 large population and the truncation scheme (which is not the same as the
 exponential weights introduced by Baluja) will have this property.}
 \yohe{I have stated in example 3 that the truncation weight is different
 from the one introduced by Baluja. }
\end{itemize}
}

\begin{abstract}
{\em Information-Geometric Optimization} (IGO) is a unified framework
of stochastic algorithms for optimization problems. Given a family of
probability distributions, IGO turns the original optimization
problem into a new maximization problem on the parameter space of the
probability distributions. IGO updates the parameter of the
probability distribution along the natural gradient, taken with
respect to the Fisher metric on the parameter manifold, aiming at
maximizing an adaptive transform of the objective function. IGO
recovers several known algorithms as particular instances: for the family
of Bernoulli distributions IGO recovers PBIL, for the family of
Gaussian distributions the pure rank-$\mu$ CMA-ES update is
recovered, and for exponential families in expectation
parametrization the cross-entropy/ML method is recovered.

This article provides a theoretical justification for the IGO framework,
by proving that any step size not greater than $1$ guarantees monotone
improvement over the course of optimization, in terms of $q$-quantile values of the objective function
$f$. The range of admissible step sizes is independent of $f$ and its
domain. We extend the result to cover the case of different step sizes
for blocks of the parameters in the IGO algorithm. Moreover, we prove
that expected fitness improves over time when fitness-proportional
selection is applied, in which case the RPP algorithm is
recovered.%\del{$q$-quantile value of the original objective is guaranteed to monotonically improve at each iteration.}
\yohe{I don't think we need to spell out the full name of CMA-ES, PBIL and RPP in the abstract. (Reviewer's comment)}

\end{abstract}

%%A category including the fourth, optional field follows...
%\category{D.2.8}{Software Engineering}{Metrics}[complexity measures, performance measures]

\yann{What are we supposed to put into "General terms"?}
\yohe{When submitting a paper to GECCO, they require to put some words out of 16 keywords defined by ACM (e.g., Theory, Algorithm, ..., there is an instruction page at ACM website.)}

 % NOT required for Proceedings

\yohe{I prefer to use ``step size'' for IGO but ``learning rate'' for PBIL or CMA-ES as they are used in their contexts.}

\yohe{As far as I know, there isn't a paper that mentions quantile improvement in CE/ML, although it seems quite natural to consider it when learning rate equals one in CE.}

\section{Introduction}

Information-Geometric Optimization (IGO) \cite{Arnold2011arxiv} is a
unified framework of model based stochastic search algorithms for any
optimization problem. As typified by Estimation of Distribution
Algorithms (EDA) \cite{Larranaga2002book}, model based randomized search
algorithms build a statistical model $P_{\theta}$ on the search space
$\X$ to generate search points. The parameters $\theta$ of the
statistical model are updated over time so that the probability
distribution hopefully concentrates around the minimum of the objective function.
In most model based algorithms such as EDAs and Ant Colony Optimization
(ACO) algorithms \cite{Dorigo1996tsmc}, parameter calibration is based on the
maximum likelihood principle or other intuitive ways. IGO, unlike
them, performs a natural gradient ascent of $\theta$ in the parameter space
$\Theta$, having first adaptively transformed the objective function
into a function on $\Theta$. %\del{The IGO updates the parameter $\theta^t$ along the natural
%gradient of a transformation of the objective function based on
%$P_{\theta}$ which is a function on $\Theta$.} 
This construction
offers maximal robustness guarantees with respect to changes in the
representation of the problem (change of parametrization of the search
space, of the parameter space, and of the fitness values).

Importantly, the IGO framework recovers several known algorithms
\cite[Section 4]{Arnold2011arxiv}. When
IGO is instantiated using the family of Bernoulli distributions on
$\{0, 1\}^d$, one obtains the
\emph{population based incremental learning} (PBIL) algorithm
\cite{Baluja1995icml}. When using the family of Gaussian distributions on
$\R^d$, IGO instantiates as a variant of \emph{covariance matrix adaptation
evolution strategies} (CMA-ES), the so-called pure rank-$\mu$ CMA-ES
update
\cite{Hansen2003ec}.  Moreover, when using an
exponential family with the expectation parameters, the IGO instance is
equivalent to the cross-entropy method for optimization
\cite{Boer2005aor}. Of course, the IGO framework not only 
provides information-theoretic derivations for
existing algorithms but automatically offers new algorithms for possibly
complicated optimization problems. For instance, the IGO update rule for
the parameters of restricted Boltzmann machines has been derived
\cite{Arnold2011arxiv}. 

Theoretical justification of the IGO framework, therefore, is important
both to provide a theoretical basis for the recovered algorithms and to
make the design principle for future algorithms more reliable.
Here we focus on providing a measure of ``progress'' over the course of
IGO optimization, in terms of quantile values of the objective function.

Parameter updates by gradient ascent are somewhat justified in general,
at least for infinitesimally small steps, because the gradient points
to the direction of steepest ascent of a function.  However, this
argument does not apply to the IGO algorithm: as the objective function is
adaptively transformed 
in a time-dependent way, the function on
which the gradient is computed changes over time, so that its increase does not
necessarily mean global improvement.
Still, the IGO framework comes with a guarantee that an
infinitesimally small IGO step along the natural gradient leads to
monotone improvement of a specified quantity, for any objective function
$f$ \cite[Proposition~5]{Arnold2011arxiv}:
%\del{. The notion of $q$-quantile
%value of $f$ has been introduced to measure improvement.  A} 
a result from
\cite{Arnold2011arxiv} is that the \emph{$q$-quantile value} of the
objective function monotonically
improves along the natural gradient. This result is limited to
the exact IGO flow, i.e., an infinite number of sample points is
considered and the step size of the gradient ascent is infinitesimal.
Still this ensures that the randomized algorithm with large sample size 
%\del{finite samples} 
stays close to the deterministic trajectory with infinite samples with
high probability, provided the step size is sufficiently small. Now the
question arises whether actual, non-infinitesimal step sizes still ensure
monotone $q$-quantile improvement. 

\yohe{To reflect a reviewer's comment, I may add comments on why it is important to study a range of appropriate step-size.}

In this article, we prove that \emph{any} step size not greater than $1$
guarantees monotone $q$-quantile value improvement in the IGO algorithm for an
exponential
family with a finite step size (Theorem~\ref{thm:qimp}), thus extending the previous result from
infinitesimal steps with continuous time to
more realistic algorithmic situations.
For instance, this ensures monotone $q$-quantile improvement in PBIL
(using uniform weights, see below), or
in the cross-entropy method for exponential families in expectation
parameters.
%\del{This ensures monotone $q$-quantile improvement in PBIL, for
%instance.}
Interestingly, our
results show that the admissible step sizes in IGO are
\emph{independent of
the objective function $f$}, at least for large population sizes (this stems from the many invariance
properties built into IGO).

We further extend the result by defining \emph{blockwise} updates in
IGO where different blocks of parameters are adjusted one after another with
different step sizes. Our motivation is that 
in practice the pure
rank-$\mu$ update CMA-ES updates the mean vector and the covariance
matrix with different learning rates. We show that the blockwise update
rule recovers the pure rank-$\mu$ CMA-ES update using different learning
rates for the mean vector and the covariance matrix
(Proposition~\ref{prop:igo-cma}). We prove that
\emph{any} distinct step sizes less than $1$ guarantee monotone
$q$-quantile improvement, which justifies the parameter setting used for
the CMA-ES in practice (Theorem~\ref{thm:iter-qimp}).

Other examples fitting into this framework
 are the Relative Payoff Procedure (\new{RPP}, also
known as expectation-maximization for reinforcement learning)
\cite{Hinton1989ai,Dayan1997nc}, or situations where fitness-proportional
selection is applied using exponential families (Theorem~\ref{thm:fitness-imp}).  The RPP is considered as
an alternative to gradient based methods that allows to use relatively
large learning rates. As it turns out, the RPP can be described as a natural
gradient based algorithm with step size $1$, and our result is an
extension of the proof of its monotone improvement to generic natural
gradient algorithms.

The article is organized as follows. In Section~\ref{sec:igo}, we explain
the IGO framework and its implementation in practice. IGO-maximum
likelihood (IGO-ML), a variant of IGO as a maximum likelihood, is
presented, followed by the relation between the IGO algorithm,
IGO-ML and the cross-entropy method for optimization, for exponential
families of distributions. In Section~\ref{sec:qi}, we prove monotone
$q$-quantile improvement in IGO-ML. The result is extended by
defining blockwise IGO-ML, and $q$-quantile improvement in
blockwise IGO-ML is proved. We also provide a result with finite but
large population sizes. Section~\ref{sec:fit} is devoted to the natural
gradient algorithm with fitness-proportional selection scheme, where
monotone improvement of expected fitness is proven. A short discussion in
Section~\ref{sec:disc} closes the article.

\section{Information-Geometric \\ Optimization}
\label{sec:igo}

In this article, we consider an objective function $f: \X \to \R$ to be minimized over any search space $\X$. The search space $\X$ may be continuous or discrete, finite or infinite. 

Let $\{P_\theta\}$ be a family of probability distributions \new{on $\X$} parametrized
by $\theta \in \Theta$ and let $p_\theta$ be the probability density function
induced by $P_\theta$ w.r.t.\ an arbitrary reference measure $\RM$ on
$\X$, namely, $P_{\theta}(\RM)=p_\theta(x)\RM$. \yann{We could also
drop $P_\theta$ altogether and use only $p_\theta$. Still, measure theory is the correct
framework here...} Given a family of probability distributions, IGO
\cite{Arnold2011arxiv} evolves the
probability distribution $P_{\theta^t}$ at each time $t$ so that higher
probabilities are assigned to better regions. To do so, IGO
transforms the objective function $f(x)$ into a new one $\Wft(x)$,
defines a function on $\Theta$ to be maximized: $\Jt{\theta} \deq
\E_{P_\theta}[\Wft(x)]$, and performs the steepest gradient ascent of
$\Jt{\theta}$ on $\Theta$. Hopefully, after some time the distribution
$P_{\theta^t}$ concentrates around minima of the objective function.

IGO is designed to exhibit as many invariance properties as possible
\cite[Section 2]{Arnold2011arxiv}.
The first property is invariance under strictly increasing
transformations of $f$. For any strictly increasing $g$, IGO
minimizes $g \circ f$ as easily as $f$. This property is realized by a
quantile based mapping of $f$ to $\Wft$ at each time. The second property
is invariance under a change of coordinates in $\X$, provided that this
coordinate change globally preserves the family of probability
distributions $\{P_\theta\}$. For example, the IGO algorithm for Gaussian
distributions on $\R^d$ is invariant under any affine transformation of
the coordinates whereas the IGO algorithm for isotropic Gaussian 
distribution is
only invariant under any translation and rotation. Invariance under
$\X$-coordinate transformation is one of the key properties for the success of the
CMA-ES. The last property is invariance under reparametrization of
$\theta$. At least for infinitesimal steps of the gradient ascent, 
IGO follows the same trajectory on the parameter space whatever the
parametrization for $\theta$ is. This property is obtained by
considering the intrinsic (Fisher) metric on the parameter space $\Theta$
and defining the steepest ascent on $\Theta$ w.r.t.\ this
metric, i.e., by using a natural gradient.

The study of the intrinsic metric on the parameter space of the
probability distribution, called a {\em statistical manifold}, is the
main topic of {\em information geometry} \cite{Amari2000book}. The most
widely used
%\del{accepted} 
divergence between two points on the space of probability distributions is the Kullback--Leibler divergence (KL divergence)
\begin{equation*}
 \KL{P_\theta}{P_{\theta'}} \deq \int \ln
 \frac{p_{\theta}(x)}{p_{\theta'}(x)} \,P_\theta(\dx) \enspace.
\end{equation*}
The KL divergence is, by definition, independent of the parametrization
$\theta$. Let $\theta' = \theta + \dtheta$. Then, the KL divergence
between $P_\theta$ and $P_{\theta + \dtheta}$ expands \cite{Kullback} as
\begin{equation}
 \KL{P_\theta}{P_{\theta+\dtheta}} = \onehalf \dtheta^\T \FM_\theta \dtheta + O( \norm{\dtheta}^3 ) \enspace,
\label{eq:KL-expansion}
\end{equation}
where $\norm{\cdot}$ is the Euclidean norm and $\FM_\theta$ is the Fisher information matrix at $\theta$ defined as
\begin{align*}
(\FM_\theta)_{ij} 
\deq & \int \frac{\partial \ln p_\theta(x)}{\partial \theta_i}\frac{\partial
\ln p_\theta(x)}{\partial \theta_j} \,P_\theta(\dx) \\
= & -\! \int \frac{\partial^2 \ln p_\theta(x)}{\partial \theta_i \partial
\theta_j} \,P_\theta(\dx) \enspace.
\end{align*}
The expansion \eqref{eq:KL-expansion} follows from the well-known
fact that the Fisher information matrix is the Hessian of KL
divergence. %\yann{Curvature is really something else...}
By using the KL divergence, we have the following property of the steepest ascent direction (see \cite{Amari1998nc}, Theorem~1, or \cite{Arnold2011arxiv}, Proposition~1).

\begin{statement}\label{sta:sad}
 Let $g$ be a smooth function on the parameter space $\Theta$. Let
 $\theta \in \Theta$ be a nonsingular point where $\nabla_{\!\theta}
 g(\theta) \neq 0$. Then the steepest ascent direction of $g$ is given by
 the so-called natural gradient $\tnabla_{\!\theta} g(\theta) \deq \FM_\theta^{-1} \nabla_{\!\theta} g(\theta)$. More precisely,
\begin{equation*}
 \frac{ \tnabla_{\!\theta} g(\theta) }{ \norm{\tnabla_{\!\theta} g(\theta)} } 
 = \lim_{\epsilon \to 0} \frac{1}{\epsilon} \hphantom{X} \argmax_{\mathclap{\substack{\dtheta \text{ such that }\\ \KL{P_\theta}{P_{\theta+\dtheta}} \leq \epsilon^2 / 2}}} \hphantom{X}  g(\theta + \dtheta) \enspace.
\end{equation*}
\end{statement}
Since KL divergence does not depend on parametrization, the natural
gradient is invariant under reparametrization of $\theta$. Hence, the
natural gradient step---steepest ascent step w.r.t.\ the Fisher
metric---is invariant at least for an infinitesimal step size
\cite[Section~2.4]{Arnold2011arxiv}.

\subsection{Algorithm Description}

For completeness, we include here a short description of the IGO
algorithm. We refer to \cite{Arnold2011arxiv} for a more complete
presentation.

First, IGO transforms the objective function into an adaptive
weighted preference by a quantile based approach. This
results in a rank based algorithm, invariant under
increasing transformations of the objective function. Define the lower and upper $P_\theta$-$f$-quantiles of $x \in \X$ as
\begin{align*}
 \qm(x) &\deq P_\theta[y: f(y) < f(x)] \\
 \qp(x) &\deq P_\theta[y: f(y) \leq f(x)] \enspace. 
\end{align*}
The lower quantile value $\qm(x)$ is the probability of sampling strictly
better points than $x$ under the current distribution $P_\theta$, while the
upper quantile value $\qp(x)$ is the probability of sampling points
better than or equivalent to
$x$. Given a weight
function (selection scheme) $w: [0, 1] \to \R$ that is non-increasing, the weighted preference $\Wf(x)$ is defined as 
\begin{equation}
 \Wf(x) \deq \begin{cases}
	   w( \qp(x) ) & \text{if } \qm(x) = \qp(x), \\
	   \frac{1}{\qp(x) - \qm(x)} \int^{\qp(x)}_{\qm(x)} w( u ) \,\rmd u & \text{otherwise}.
	  \end{cases}
\label{eq:wf}
\end{equation}
This way, the quality of a point is measured by a function of the
$P_\theta$-quantile in which it lies.
A typical choice of the selection scheme $w$ is $w(u) =
\indicator{u \leq q} / q$, $0 < q < 1$. We call it the
\emph{$q$-truncation} selection scheme. Using $q$-truncation
amounts, in the final IGO algorithm, to giving the same positive weight
to a fraction $q$ of the best samples in a population, and weight $0$ to
the rest, as is often the case in practice.

Next, IGO turns the original objective function $f$ on the search space
$\X$ into a function $\Jt{\theta}$ on the statistical manifold $\Theta$ by defining
\begin{equation}
 \Jt{\theta} \deq \E_{P_\theta}[ \Wft(x) ] \enspace.
\label{eq:Jt}
\end{equation}
Note that $\Jt{\theta}$ depends on the current position $\theta^t$. Then, the gradient of $\Jt{\theta}$ is computed as 
\begin{equation}
\begin{split}
 \nabla\!_\theta \Jt{\theta} 
    &= \nabla\!_\theta \,\E_{P_\theta}[ \Wft(x) ] \\
    &= \nabla\!_\theta \!\int \Wft(x) \,p_\theta(x) \, \RM \\
    &= \int \Wft(x) \,p_\theta(x) \,\nabla\!_\theta \ln p_\theta(x) \,\RM \\
    &= \E_{P_\theta}[ \Wft(x) \, \nabla\!_\theta \ln p_\theta(x) ] \enspace.
\end{split} 
\label{eq:grad-J}
\end{equation}
Here we have used the relation $\nabla p_\theta(x) = p_\theta(x) \nabla\! \ln p_\theta(x)$.

Finally, IGO uses natural gradient ascent on the parameter space. The natural gradient on the statistical manifold $(\Theta, \FM)$ equipped with the Fisher metric $\FM$ is given by the product of the inverse of the Fisher information matrix, $\FM_{\theta}^{-1}$, and the vanilla gradient. That is, the natural gradient of $\Jt{\cdot}$ at $\theta$ is written as $\tnabla_{\!\theta} \Jt{\theta} = \FM_{\theta}^{-1} \nabla_{\!\theta} \Jt{\theta}$. According to \eqref{eq:grad-J}, we can rewrite the natural gradient as
\begin{equation}
 \tnabla_{\!\theta} \Jt{\theta} = \E_{P_\theta}[ \Wft(x) \,\tnabla\!_\theta \ln p_\theta(x) ] \enspace.
\label{eq:ng}
\end{equation}
Introducing a finite step size $\deltat$, IGO finally updates the parameter as follows
\begin{equation}
 \theta^{t+\deltat} = \theta^{t} + \deltat \,\tnabla\!_\theta \Jt{\theta} \bigr\rvert_{\theta = \theta^t} \enspace.
\label{eq:igo}
\end{equation}

\subsection{Implementation and Recovering \\ Algorithms}

When implementing IGO in practice, it is necessary to estimate the
expectation in \eqref{eq:ng}. The approximation is done by the Monte
Carlo method using $\lambda$ samples taken from $P_{\theta^t}$. Let $x_1,\dots, x_\lambda$ be independent samples from $P_{\theta^t}$. 

First, we need to approximate $\Wft(x_i)$ for each $i = 1,\dots, \lambda$. Define 
\begin{align*}
 \rkm(x_i) &\deq \#\{j, f(x_j) < f(x_i)\} \\
 \rkp(x_i) &\deq \#\{j, f(x_j) \leq f(x_i)\} \enspace,
\end{align*}
let
\begin{equation*}
 \barw_i \deq \int_{(i-1)/\lambda}^{i/\lambda} w(q) \,\rmd q, \quad
 \forall i \in \llbracket1, \lambda\rrbracket,
\end{equation*}
and set
\begin{equation}
  \wi = \frac{1}{\rkp(x_i) - \rkm(x_i)} \sum_{j=\rkm(x_i) + 1}^{\rkp(x_i)} \barw_j \enspace.
\label{eq:wi}
\end{equation}
Then $\lambda \wi$ is a consistent estimator of $\Wft(x_i)$, in other words,
$\lim_{\lambda \to \infty} \lambda \wi = \Wft(x_i)$ with probability one.
(See the proof of Theorem~4 in \cite{Arnold2011arxiv}.) 
%\del{Since the values
%of $\rkm$ and $\rkp$ are integers in $\llbracket0, \lambda-1\rrbracket$
%and $\llbracket1, \lambda\rrbracket$ respectively, letting 
%\begin{equation*}
% \barw_i \deq \int_{(i-1)/\lambda}^{i/\lambda} w(q) \,\rmd q, \quad
% \forall i \in \llbracket1, \lambda\rrbracket,
%\end{equation*}
%we can rewrite \eqref{eq:wi} as
%\begin{equation*}
%  \wi = \frac{1}{\rkp(x_i) - \rkm(x_i)} \sum_{j=\rkm(x_i) + 1}^{\rkp(x_i)} \barw_j \enspace.
%\end{equation*}
%}
If there are no ties in our sample, i.e.\ $f(x_i) \neq
f(x_j)$ for any $i \neq j$, then $\rkp(x_i) = \rkm(x_i) + 1$ and
\eqref{eq:wi} simply reads $\wi = \barw_{\rkp(x_i)}$, but \eqref{eq:wi}
is a mathematically neater definition of rank based weights 
accounting for possible ties.
In practice we just design the $\lambda$ weight values
$\barw_1,\ldots,\barw_\lambda$, instead of the selection
scheme $w$. 

In the rest of this article, we assume for simplicity that the
selection weights $\barw_i$ are non-negative and sum to $1$. This is the
case, for instance, if the selection scheme $w$ is $q$-truncation as
above.

% \begin{equation}
%  \wi \deq \frac{1}{\rkp(x_i) - \rkm(x_i)} \int_{\rkm(x_i) /
%  \lambda}^{\rkp(x_i) / \lambda} w(q) \,\rmd q \enspace.
% \label{eq:wi}
% \end{equation}
% Then, $\lambda \wi$ is a consistent estimator of $\Wft(x_i)$, that is,
% $\lim_{\lambda \to \infty} \lambda \wi = \Wft(x_i)$ with probability one.
% (See the proof of Theorem~4 in \cite{Arnold2011arxiv}.) Since the values
% of $\rkm$ and $\rkp$ are integers in $\llbracket0, \lambda-1\rrbracket$
% and $\llbracket1, \lambda\rrbracket$ respectively, letting 
% \begin{equation*}
%  \barw_i \deq \int_{(i-1)/\lambda}^{i/\lambda} w(q) \,\rmd q, \quad
%  \forall i \in \llbracket1, \lambda\rrbracket,
% \end{equation*}
% we can rewrite \eqref{eq:wi} as
% \begin{equation*}
%   \wi = \frac{1}{\rkp(x_i) - \rkm(x_i)} \sum_{j=\rkm(x_i) + 1}^{\rkp(x_i)} \barw_j \enspace.
% \end{equation*}
% So \new{in practice we just design} $\lambda$ weight values $\barw_i$, instead of the selection
% scheme $w$. If there are no ties in our samples, i.e.\ $f(x_i) \neq
% f(x_j)$ for any $i \neq j$, then $\rkp(x_i) = \rkm(x_i) + 1$ and
% \eqref{eq:wi} simply reads $\wi = \barw_{\rkp(x_i)}$. \eqref{eq:wi}
% is a mathematically neater definition of rank based weights 
% accounting for possible ties.

Next, Monte Carlo sampling is applied to the expectation \eqref{eq:ng}, using
$\wi$ and $x_i$. Replacing the expectation with a sample average
$\frac{1}{\lambda} \sum_{i=1}^{\lambda}$ and $\Wft(x_i)$ with $\lambda
\wi$, we get
\begin{equation}
 \Gt \deq \sum_{i=1}^{\lambda} \wi \,\tnabla\!_\theta \ln p_\theta(x_i) \rvert_{\theta = \theta^t} \enspace.
\label{eq:Gt}
\end{equation}
Again, $\Gt$ is a consistent estimator of the IGO step at
$\theta^t$, i.e., of $\tnabla\!_\theta \Jt{\theta} \lvert_{\theta = \theta^t}$. See Theorem~4 in \cite{Arnold2011arxiv}.

Now the practical IGO algorithm implementation can be written in the form of a black-box search algorithm as
\begin{enumerate}
 \item Sample $x_i$, $i = 1, \dots, \lambda$, independently from $P_{\theta^t}$;
 \item Evaluate $f(x_i)$ and compute $\rkp(x_i)$ and $\rkm(x_i)$;
 \item Evaluate $\Gt =\sum_{i=1}^{\lambda} \wi
 \,\tnabla\!_\theta \ln p_\theta(x_i) \rvert_{\theta = \theta^t}\enspace$;
 \item Update the parameter: $\theta^{t+\deltat} = \theta^{t} + \deltat \cdot \Gt$.
\end{enumerate}

Finally, to obtain an explicit form of the parameter update equation, we
need to know the explicit form of the natural gradient of the
log-likelihood, which depends on a family of probability distributions
and its parametrization. Explicit forms of $\tnabla\!_\theta \ln
p_\theta(x)$ are known for some specific families of probability
distributions with specific parametrizations, and the above
algorithm sometimes coincides with several known algorithms.

\begin{example}\label{ex:pbil}
The family of Bernoulli distributions on $\X = \{0, 1\}^d$ is defined as
$P_\theta(x) = \prod_{j=1}^{d} \theta_j^{x_j} (1 - \theta_j)^{1 - x_j}$.
The natural gradient of the log-likelihood is readily computed as
$\tnabla\!_\theta \ln p_\theta(x) = x - \theta$ (Section~4.1 in \cite{Arnold2011arxiv}). The natural gradient update reads
\begin{equation*}
 \theta^{t+\deltat} = \theta^t + \deltat \sum_{i=1}^{\lambda} \wi \left( x_i - \theta^t \right) \enspace.
\end{equation*}
This is equivalent to so-called PBIL (population based incremental
learning, \cite{Baluja1995icml}).%\del{ for the natural choice of weight which
%sets $\wi$ to $1$ for the $\mu$ best points.} 
See Section~4.1 in \cite{Arnold2011arxiv} for details. 
\end{example}

\begin{example}\label{ex:cmaes}
 The probability density function of a multivariate Gaussian distribution
 on $\X = \R^d$ with mean vector $m$ and covariance matrix $C$, is
 defined as 
\[
 p_\theta(x) = \bigl(\det( 2\pi C ) \bigr)^{-1/2} \exp\bigl(
 - (x - m)^\T C^{-1} (x - m) / 2 \bigr) \enspace.
\]
When $\theta = (m, C)$, the
   explicit form of $\tnabla \ln p_\theta(x)$ is known to be
$\tnabla \ln p_\theta(x) = \bigl[
\begin{smallmatrix}
 x - m \\  (x - m)(x - m)^\T - C
\end{smallmatrix}\bigr]$ (see \cite{Akimoto2010ppsn}). Then, the natural gradient update reads
\begin{equation*}
  \theta^{t+\deltat} = \theta^t + \deltat \sum_{i=1}^{\lambda} \wi 
   \begin{bmatrix}
    x - m^t \\  
    (x - m^t)(x - m^t)^\T - C^t
   \end{bmatrix} \enspace.
\end{equation*}
This is equivalent to the pure rank-$\mu$ CMA-ES update \cite{Hansen2003ec}
\begin{align*}
 m^{t+1} &= \textstyle m^t + \eta_m \sum_{i=1}^{\lambda} \wi (x_i - m^t) \\
 C^{t+1} &= \textstyle C^t + \eta_C \sum_{i=1}^{\lambda} \wi \bigl( (x_i - m^t)(x_i - m^t)^\T - C^t \bigr)
\end{align*}
except that $\eta_m = \eta_C = \deltat$ in the natural gradient update.
\end{example}

\subsection{Maximum likelihood, IGO-ML, and cross-entropy}

In the sequel, we prove monotone improvement of the objective
function for a variant of IGO known as IGO-maximum likelihood
(IGO-ML, introduced in \cite[Section~3]{Arnold2011arxiv}). The result is then transferred to IGO because the two
algorithms exactly coincide in an important class of cases, namely,
exponential families using mean value parametrization.

The IGO-ML algorithm \cite[Section~3]{Arnold2011arxiv} updates the current parameter value $\theta^t$ by taking
a weighted maximum likelihood of the current distribution and the best
sampled points. Assume as above that $\sum \wi=1$. Then the \emph{IGO-ML
update}
is defined as
\begin{multline}
\label{eq:IGOML}
\theta^{t+\deltat}=\argmax_\theta
\Bigg\{
{\textstyle(1-\deltat)}\,
\E_{P_{\theta^t}} \left[ \ln p_\theta(x)\right]
\\+\,\deltat
\sum_i \wi \ln p_{\theta}(x_i)
\Bigg\}
\end{multline}
where we note that the first part is the cross-entropy of
$P_{\theta^t}$ and $P_\theta$, and thus, taken alone, is maximized for
$\theta=\theta^t$. 
Taking the limit $\lambda\to\infty$, we also define
the \emph{infinite-population IGO-ML update} as
%\yann{Should we replace "ideal" with the more
%explicit "infinite-population"? This way readers wouldn't have to go back
%to the definition to see what it means in the theorems...}
%\yohe{Yes. It is better to use ``infinite-population'' than ``ideal''.}
\begin{equation}
\label{eq:ideal-IGOML}
\theta^{t+\deltat}=\argmax_\theta
\Big\{
{\textstyle(1-\deltat)}\,
\E_{P_{\theta^t}} \left[ \ln p_\theta(x)\right]
+\,\deltat \Ht(\theta)
\Big\}
\end{equation}
where we set
\[
\Ht(\theta) \deq \E_{P_{\theta^t}} \bigl[\Wft(x)  \ln p_\theta(x) \bigr]
\]
a ``weighted cross-entropy'' of $\theta$ and $\theta^t$. 

Note that the finite- and infinite-population IGO-ML updates only make sense when there
is a unique maximizer $\theta$ in \eqref{eq:IGOML} and 
\eqref{eq:ideal-IGOML}, respectively. This assumption is always satisfied,
for instance, for exponential families of probability distributions, as
considered below (Statement~\ref{stat:unique-ML}).

The IGO-ML update is compatible with the IGO update, in the sense that
for $\deltat\to 0$ the direction and magnitude of these updates coincide
\cite[Section~3]{Arnold2011arxiv}.

The IGO-ML method is also related to the cross-entropy (CE) or
maximum-likelihood (ML)
method for optimization \cite{Boer2005aor}, which can be written as
\begin{equation*}
 \theta^{t+1} = \argmax_\theta \sum_{i=1}^{\lambda} \wi \ln p_\theta(x_i) 
\end{equation*}
and its smoothed version which reads \cite{Boer2005aor}
\begin{equation}
 \theta^{t+\deltat} = (1-\deltat) \theta^t + \deltat \argmax_\theta \sum_{i=1}^{\lambda} \wi \ln p_\theta(x_i) \enspace.
\label{eq:smoothed-ce}
\end{equation}
Note that IGO-ML is parametrization-independent whereas for
$\deltat\neq 1$ the smoothed
CE/ML method is not. Consequently, in general these updates will
differ.

\subsection{IGO and IGO-ML for Exponential Families}
\label{sec:IGO-Exp}
An {\em exponential family} is a set $\{ p_\theta; \theta \in \Theta\}$
of probability density functions $p_\theta$ with respect to an arbitrary measure $\RM$ on $\X$ defined as
\begin{equation}
 p_{\theta}(x) = \frac{1}{Z(\theta)}\exp\Biggl( \sum_{i=1}^{n} \np_{i}(\theta) T_{i}(x) \Biggr) \enspace,
\label{eq:exponential}
\end{equation}
where $\np = (\np_{i})_{1 \leq i \leq n}$ is the so-called {\em natural}
(i.e.\ canonical) {\em parameter}; each $T_i$, ${1 \leq i \leq n}$ is a
map $T_i: \X \to \R$ such that $\{T_{1}, \dots, T_{n}, x \mapsto 1\}$
are linearly independent; $Z(\theta)$ is the normalization factor. 
This linear independence ensures that the manifold of the
exponential family is nonsingular. Many probability models, including
multivariate Gaussian distributions, are expressed as exponential
families. See \cite[Section~2.3]{Amari2000book} for examples.

If we define
\begin{equation}
 \ep(\theta) \deq \E_{ P_\theta }[ T(x) ] = \int T(x) \,p_\theta(x) \,\RM\enspace,
\label{eq:ep}
\end{equation}
$\ep = (\ep_i)_{1 \leq i \leq n}$ is the so-called {\em expectation
parameter}. For example, the expectation parameter for the multivariate
Gaussian distribution encodes the first moment $\E_{P_\theta}[x]$ and
the second moment $\E_{P_\theta}[x x^\T]$. Other examples can be found
in \cite[Section~3.5]{Amari2000book}.

%\todo{?Cite papers by Malago about natural gradient of exponential
%families}

We will repeatedly and implicitly make use of the following well-known fact for
exponential families\todo{ ref?}.

\begin{statement}
\label{stat:unique-ML}
Let $x_1,\ldots,x_k$ be $k$ points in $X$ and let
$\alpha_1,\ldots,\alpha_k$ be non-negative numbers with $\sum
\alpha_i=1$.
Then the value $\theta$ of the parameter such that the associated
expectation parameter satisfies
$\eta(\theta)=\sum \alpha_i T(x_i)$, if it belongs
to the statistical manifold, is the unique maximizer of the weighted
log-likelihood:
$\theta=\argmax \sum \alpha_i \ln p_\theta(x_i)$.
An analogous statement holds if the finite sum is replaced with an
integral or a combination of both.
\end{statement}

(Uniqueness boils down to strict concavity of
$\ln p_\theta(x)$ as a function of $\theta$. The restriction placed on $\eta$ to belong to the statistical manifold is
necessary: for instance, for Gaussian distributions, if the number of
points $k$ is not greater than the dimension of the ambient space, a
degenerate distribution $\theta$ will result.)

The following statement from \cite{Arnold2011arxiv} shows that the natural gradient of a function in the expectation parametrization is given by the vanilla gradient of the function w.r.t.\ the normal parameter, and vice versa.
\begin{statement}[Proposition~22 in \cite{Arnold2011arxiv}]
\label{sta:ng-vg}
Let $g$ be a function on the statistical manifold of an exponential family as above. Then the components of the natural gradient w.r.t.\ the expectation parameters are given by the vanilla gradient w.r.t.\ the natural parameters and vice versa, that is,
\begin{equation*}
 \tnabla_{\!\ep_i} g = \frac{\partial g}{\partial \np_i} \quad \text{ and
 } \quad  \tnabla_{\!\np_i} g = \frac{\partial g}{\partial \ep_i} \enspace.
\end{equation*}
\end{statement}

According to Statement~\ref{sta:ng-vg}, each component of the natural gradient of the log likelihood $\ln p_\theta(x)$ under the exponential parametrization $\theta = \ep$ is equivalent to each component of the vanilla gradient, i.e.,
\begin{equation}
 \tnabla_{\ep_i} \ln p_\theta(x) = \frac{\partial \ln
 p_\theta(x)}{\partial \np_i} = T_i(x) - \ep_i \enspace,
\end{equation}
where the latter equality is well-known, e.g., \cite[(2.33)]{Amari2000book}. The
IGO update \eqref{eq:igo} under the expectation parametrization thus reads
\begin{equation}
 \ep^{t+\deltat} = \ep^{t} + \deltat \, \E_{P_{\theta^t}}[ \Wft(x) (T(x) - \ep^t)]
\label{eq:igo-exp}
\end{equation}
and the natural gradient update with finite sample size reads
\begin{equation}
 \ep^{t+\deltat} = \ep^t + \deltat \sum_{i=1}^{\lambda} \wi \left( T(x_i) - \ep^t \right) \enspace.
\label{eq:igo-exp-sample}
\end{equation}

Suppose as above that the selection weights sum to one: $\E_{P_{\theta}}[
\Wf(x) ] = \int_0^1 w(q) \rmd q = 1$ and thus $\sum \wi=1$.
Then, IGO has a close relation with the CE/ML for optimization. As is
stated in Theorem~15 in \cite{Arnold2011arxiv}, for an exponential family
the CE/ML method \eqref{eq:smoothed-ce} and the IGO instance
\eqref{eq:igo-exp-sample}, when expressed with the expectation
parametrization ($\theta = \ep$), coincide with IGO-ML
\eqref{eq:IGOML}. 

\begin{statement}[Theorem 15 in \cite{Arnold2011arxiv}]
For optimization using an exponential family $\{P_\theta\}$, these three
algorithms coincide: IGO-ML; the IGO expressed in expectation
parameters; the CE/ML expressed in expectation parameters. That is,
for an exponential family with the expectation
parametrization, for $0\leq \deltat\leq 1$ we have (writing in turn IGO, CE/ML and IGO-ML)
\begin{equation}
 \begin{split}
  \theta^{t+\deltat} 
  &= \theta^t + \deltat \sum_{i=1}^{\lambda} \wi \left( T(x_i) - \theta^t \right) \\
  &= (1-\deltat)\theta^t + \deltat \argmax_\theta \sum_{i=1}^{\lambda} \wi \ln p_\theta(x_i) \\
  &= \argmax_\theta \biggl\{ (1 - \deltat) \E_{P_{\theta^t}} \left[ \ln p_\theta(x)\right]  \\
  &\qquad \qquad \qquad \qquad + \deltat \sum_{i=1}^{\lambda} \wi \ln p_\theta(x_i) \biggr\} 
  \enspace.
 \end{split}
\label{eq:igo-equiv-sample}
\end{equation}
%\yohe{Reviewer's comment: `` in statement 4 you say (writing in turn IGO, CE/ML
% and IGO-ML) which make me expect 3 equations to follow, but as far as I can 
%tell only 1 does. I am confused.'' \\
%I think it is obvious that the above three equalities correspond to those three update.}
%\yann{I fully agree with you. Hurried reviewing...}
\end{statement}

In the limit of infinite sample size $\lambda\to\infty$ this rewrites
\begin{equation}
  \begin{split}
  \theta^{t+\deltat} 
  &= \theta^t + \deltat \,\E_{P_{\theta^t}} \bigl[\Wft(x) \left( T(x) - \theta^t \right) \bigr] \\
  &= (1 - \deltat) \theta^t + \deltat \argmax_\theta \Ht(\theta) \\
  &= \argmax_\theta \Bigl\{ (1 - \deltat) \E_{P_{\theta^t}} \left[ \ln p_\theta(x)\right] + \deltat \Ht(\theta) \Bigr\}
 \end{split}
\label{eq:igo-equiv}
\end{equation}
where we recall that
$
\Ht(\theta) = \E_{P_{\theta^t}} \bigl[\Wft(x)  \ln p_\theta(x) \bigr]
$.

\begin{remark}
\label{rem:malago}
Malagò~et~al.~\cite{Malago2011foga} study information-geometric aspects
of exponential families for optimization. One difference from the IGO
framework is that the optimization problem is defined as the minimization of the expectation of the objective function over $P_\theta$, namely
\[
 \min_\theta \E_{P_\theta}[ f(x) ],
\]
which they call the stochastic relaxation of the original optimization
problem. They study this for an exponential family on a discrete search
space with the natural parametrization ($\theta = \np$) and propose the
natural gradient descent algorithm. Note that this requires computation
of the empirical Fisher information matrix to perform natural gradient
descent. However, if the algorithm is modified to use the expectation parameters instead, one can compute the natural gradient descent directly as
\begin{equation}
 \ep^{t+\deltat} = \ep^t - \deltat \sum_{i=1}^{\lambda} f(x_i) \left( T(x_i) - \ep^t \right) \enspace. 
\label{eq:malago-exp}
\end{equation}
We study this algorithm in Section~\ref{sec:fit}.
\end{remark}

\section{Quantile Improvement}
\label{sec:qi}

One possible way to provide
theoretical backing for an optimization algorithm
%\del{A simple but meaningful way to justify the framework of an optimization
%algorithm} 
is to show
monotonic improvement at each step of the algorithm
(although this
is by no means necessary: e.g., for stochastic
algorithms, this is not expected to hold at each step). 
For example, consider the
sphere function $f: x \mapsto \norm{x}^2$. Then, it is easy to show that
the gradient steps $x^{t+\deltat} = x^t - \deltat\, \nabla\!_x f(x^t)$
generate a monotonically decreasing sequence $\{ f(x^t) \}_{t \geq 0}$
provided $0 < \deltat \leq 1/2$. For any smooth function, infinitesimal
gradient steps are guaranteed to improve the objective function
values; but in general the admissible step size strongly depends on
the function and has to be adjusted by the user.

When it comes to the counterpart in IGO, however, we follow the
gradient of the 
function $\Jt{\theta}$, which depends on $\theta^t$, so that step-by-step improvement in the objective, $\Jt{\theta^{t+1}} > \Jt{\theta^t}$, does not necessarily mean improvement. (It might happen that $\Jtt{\theta^t} > \Jtt{\theta^{t+1}}$ and $\Jt{\theta^{t+1}} > \Jt{\theta^t}$ at the same time.)
\yohe{Can we give an example case? It is almost obvious that there are cases that $J(\theta_a \mid \theta_b) > J(\theta_b \mid \theta_b)$ and $J(\theta_b \mid \theta_a) \geq J(\theta_a \mid \theta_a)$ for some $\theta_a$ and $\theta_b$ (because the contradiction implies that for any $\theta_a$ and $\theta_b$ there is an increasing function $g$ such that $J(\cdot \mid \theta_a) = g( J(\cdot \mid \theta_b))$, which is unbelievable.)}

A key feature of the IGO framework is its invariance under
changing the objective function $f$ by an increasing transformation
(e.g.\ optimizing $f^3$ instead of $f$). Thus, any measure of progress
that is not compatible with such transformations (e.g.\ the expectation
$\E_{P_{\theta^t}} f$) is not a good candidate to always improve over the course of IGO
optimization.

As a measure of improvement, Arnold~et~al.~\cite{Arnold2011arxiv} use the
notion of \emph{$q$-quantile of $f$}. The $q$-quantile $\QPq(f)$ of $f$
under a probability distribution $P$ is any number $m$ such
that $P[x: f(x) \leq m] \geq q$ and $P[x: f(x) \geq m] \geq 1 - q$. For
instance, $\QPq(f)$ is the median value of $f$ under $P$ if $q = 1/2$.
For smooth distributions and continuous $f$ there is only one such number $m$, but in
general the set of such $m$ may be a closed interval, for instance if $f$
has ``jumps''. For the sake of
definiteness let us use the largest such value:
\begin{align*}
\QPq(f)\deq \sup \big\{ m\in \R :{}& P[x: f(x) \leq m] \geq q 
\\& \text{and } P[x:
f(x) \geq m] \geq 1 - q \big\}
\enspace.
\end{align*}
(This is because we want to minimize the objective function $f$; when
IGO is used for maximization instead, Theorem~\ref{thm:qimp} has to be
written using an infimum in the definition of $\QPq(f)$ instead.)

It is proven in \cite{Arnold2011arxiv} that when using the $q$-truncation
selection scheme, the $q$-quantile value of $f$ monotonically decreases
along infinitesimal IGO steps.

\begin{statement}
[Proposition~5 in \cite{Arnold2011arxiv}]
 Consider the $q$-truncation selection scheme $w(u) = \indicator{u \leq q} / q$ where $0 < q < 1$ is fixed. Then each infinitesimal IGO step \eqref{eq:igo} where $\deltat$ is infinitesimal leads to monotonic improvement in the $q$-quantile of $f$: $\QPttq(f) \leq \QPtq(f)$.
\end{statement}

\subsection{Quantile Improvement in IGO-ML}

In practice, explicit algorithms do not use continuous time with
infinitesimal time steps: the time step $\deltat$ may be quite
large and its calibration may be an important issue.
It is more interesting and important to see how long steps we can take
along the natural gradient, i.e.\ how large a $\deltat$ we can choose while
guaranteeing $q$-quantile improvement. 
%\del{When we focus on an
%exponential family with the expectation parametrization, we can obtain a
%stronger conclusion.
%
%More generally, we consider the quantile improvement in IGO-ML, which includes the IGO algorithm and the CE/ML method under an exponential family with the exponential parametrization as is mentioned in Section~\ref{sec:IGO-Exp}}

When using IGO-ML (and thus when using IGO
or CE/ML on
an exponential family with the expectation parametrization), we can obtain such a
conclusion; the size of the steps may even be chosen independently
of the objective function.

%\yann{We repeatedly use that $\theta=\theta' \Leftrightarrow
%p_\theta=p_{\theta'}$, which is only true if the parametrization is
%one-to-one.}
%\yohe{Yes. I have used some properties of exponential families. }

%\yohe{If we use the selection scheme: $\Wft(x) = \indicator{ \qpt(x) < q} / q$, we can guarantee that $P_{\theta^{t+\deltat}}[x: f(x) < m] > P_{\theta^{t}}[x: f(x) < m]$ as long as the mess of $\{\qpt(x) < q\}$ is positive. }

\begin{theorem}
\label{thm:qimp}
%Let $\{P_\theta\}$ be a family of exponential distributions
%\eqref{eq:exponential} parametrized by the expectation parameter
%\eqref{eq:ep}. \todo{Use IGO-ML instead.}
Let the selection scheme be $w(u) = \indicator{u \leq q} / q$ where $0 <
q < 1$. Assume that the $\argmax$ defining the IGO-ML step
\eqref{eq:ideal-IGOML} is uniquely 
determined.
Then for $0 < \deltat \leq 1$, each infinite-population IGO-ML step \eqref{eq:ideal-IGOML} leads to
$q$-quantile improvement: $\QPttq(f) \leq \QPtq(f)$.

Moreover, equality can hold only if
$P_{\theta^{t+\deltat}}=P_{\theta^t}$ or if $P_{\theta^{t+\deltat}}[x: f(x) =
\QPtq(f)] > 0$.
\end{theorem}

%\yann{Since this is apparently not known for CE/ML, we have nothing to
%lose by being explicit.}

\begin{corollary}
For exponential families written in expectation parameters, on any search
space, the same
holds for the CE/ML method and for the IGO algorithm.
\end{corollary}

Note that the first condition for equality means the algorithm has
reached a stable point.

The second condition for equality typically happens for discrete search
spaces: on such spaces, the $q$-quantile evolves in time by discrete
jumps even when $\theta^t$ moves smoothly, so we cannot expect strict
quantile improvement at each step. On the other hand, 
with continuous distributions on continuous search spaces, the second
equality condition can only occur if the objective function has a plateau
(a level set with non-zero measure).

%\todo{Discuss strict improvement for continuous case: Second theorem with
%strict improvement when the quantiles are sharp (also uses properties of
%exp families).}

%\yohe{(24/07/2012) I have revised the proof of the Lemma and I have realized we don't need any assumptions that I did such as absolute continuity of $P_\theta$.}

\begin{proof}
If $P_{\theta^{t+\deltat}} = P_{\theta^{t}}$, obviously $\QPttq(f) = \QPtq(f)$. Hereunder, we assume $P_{\theta^{t+\deltat}} \neq P_{\theta^{t}}$.

Consider the function $\Jt{\theta}$ defining the expected
$P_{\theta^t}$-adjusted fitness
of a random point under $P_\theta$:
\[
 \Jt{\theta} 
  = \E_{P_\theta}\bigl[ \Wft(x) \bigr]
\]
and remember that $\Jt{\theta^t}=1$. The idea is as follows: letting
$Y$ be the set of points with $P_{\theta^t}(Y)=q$ at which the objective
function $f$ is smallest (the sublevel set of $f$ with $P_{\theta^t}$-mass $q$),
then with our choice of $w$, $\Wft(x)$ is (up to technicalities) equal to
$1/q$ on $Y$ and 
$0$ elsewhere, so that $\Jt{\theta}$ represents $1/q$
times the $P_\theta$-probability of falling into $Y$ (hence
$\Jt{\theta^t}=1$). Thus $\Jt{\theta}>1$ will mean that the
$P_\theta$-probability of falling into $Y$ is larger than $q$, so that
$P_\theta$ improves over $P_{\theta^t}$ and the $q$-quantile has
decreased.

We are going to prove that the IGO-ML update satisfies
$\Jt{\theta^{t+\deltat}}> 1$ if $P_{\theta^t} \neq P_{\theta^{t+\deltat}}$. More precisely we prove that
\[
 \Jt{\theta^{t+\deltat}} > \exp \left(
\frac{1 - \deltat}{\deltat} \KL{P_{\theta^t}}{P_{\theta^{t+\deltat}}}
\right) \enspace.
\]
This will imply quantile improvement, thanks to the
following lemma, the proof of which is postponed.

\begin{lemma}
\label{lem:JandQ}
Let the selection scheme $w$ be as above.
If $\Jt{\theta^{t+\deltat}} > 1$, then $\QPttq(f) \leq \QPtq(f)$. 
If moreover $P_{\theta^{t+\deltat}}[x: f(x) = \QPtq(f)] = 0$, then $\QPttq(f) < \QPtq(f)$.
\end{lemma}

%\yann{I think the lemma and its proof are OK now.}

The lower bound on $\Jt{\theta^{t+\deltat}}$ is obtained as follows.
Since $\int \Wft(x) p_{\theta^t}(x) \dx = 1$ and $\Wft(x) p_{\theta^t}(x) \geq 0$ for any $x$, $\Wft(x) p_{\theta^t}(x)$ can be viewed as a probability density function. Since $\ln$ is concave, by Jensen's inequality we have
\begin{equation}
\begin{split}
\ln \Jt{\theta}
&= \ln \int \frac{p_\theta(x)}{p_{\theta^t}(x)} \Wft(x) p_{\theta^t}(x) \dx \\
&\geq \int \ln \biggl( \frac{p_\theta(x)}{p_{\theta^t}(x)} \biggr) \Wft(x) p_{\theta^t}(x) \dx \\
&= \Ht(\theta) - \Ht(\theta^t) \enspace. 
\end{split}
\label{eq:JandH}
\end{equation}
%\del{Equality holds if and only if $p_\theta(x) = p_{\theta^t}(x)$ almost everywhere on the set where $\Wft(x) p_{\theta^t}(x) \neq 0$.}\yann{only on the set
%where $\Wft(x) p_{\theta^t}(x)\neq 0$!! For $\deltat<1$ this is not a
%problem as $\KL{...}{...}$ below will be positive, but for $\deltat=1$ the equality case
%should be handled differently.}
Thus, if $\Ht(\theta) > \Ht(\theta^t)$ we have $\Jt{\theta} > 1$.
%\del{if $\Ht(\theta) \leq \Ht(\theta^t)$ we have $\Jt{\theta} \geq 1$ with equality holding if and only if $\theta = \theta^t$.}

Now, according to \eqref{eq:ideal-IGOML}, $\theta^{t+\deltat}$
uniquely maximizes the quantity $(1 - \deltat) \E_{P_{\theta^t}}
\bigl[ \ln p_\theta(x)\bigr] + \deltat \Ht(\theta)$. Therefore, if
$\theta^{t+\deltat}\neq \theta^t$, we have
\begin{multline*}
 (1 - \deltat) \E_{P_{\theta^t}} \bigl[ \ln p_{\theta^{t+\deltat}}(x)\bigr] + \deltat \Ht(\theta^{t+\deltat}) \\
 > (1 - \deltat) \E_{P_{\theta^t}} \bigl[ \ln p_{\theta^t}(x)\bigr] + \deltat \Ht(\theta^{t})
\end{multline*}
and rearranging we get
\begin{multline}
\Ht(\theta^{t+\deltat}) - \Ht(\theta^{t}) \\
 > \frac{1 - \deltat}{\deltat} \left(\E_{P_{\theta^t}} \bigl[ \ln p_{\theta^t}(x)\bigr]  -  \E_{P_{\theta^t}} \bigl[ \ln p_{\theta^{t+\deltat}}(x)\bigr] \right) 
\\ = \frac{1 - \deltat}{\deltat} \KL{P_{\theta^t}}{P_{\theta^{t+\deltat}}} \enspace.
\label{eq:HandKL}
\end{multline}
%Notice that $\E_{P_{\theta^t}} \bigl[ \ln p_{\theta}(x)\bigr]$ is maximized when $\theta = \theta^t$.
The right-hand side of this inequality is non-negative for
$0 < \deltat \leq 1$.
%\del{, and is positive unless
%$\theta^t=\theta^{t+\deltat}$ or $\deltat=1$. Thus $\Ht(\theta^{t+\deltat}) \geq
%\Ht(\theta^{t})$, hence $\Jt{\theta^{t+\deltat}} \geq 1$ with equality \del{if
%and }only if $\theta = \theta^t$.}
%Hence $\Ht(\theta^{t+\deltat}) \geq \Ht(\theta^{t})$. Since $\Ht(\theta^{t+\deltat}) \geq \Ht(\theta^{t}) \Longrightarrow \Jt{\theta^{t+\deltat}} > 1 \Longrightarrow \QPttq(f) > \QPtq(f)$ unless $\theta^{t + \deltat} = \theta^t$, we obtain $\QPttq(f) > \QPtq(f)$ for $0 < \deltat \leq 1$ as long as $\theta^{t + \deltat} \neq \theta^{t}$.

This will prove the theorem once Lemma~\ref{lem:JandQ} is proved, which
we now proceed to do.
\end{proof}

\begin{proof}[of Lemma~\ref{lem:JandQ}]
%\del{
%The set of $q$-quantiles of the function $f$ is either a singleton or a closed interval. 
%
%If this set is a singleton, we have three cases to
%consider: for any $x$ such that $f(x) = m$, $\qpt(x) = q \geq \qmt(x)$;
%or $\qpt(x) > q = \qmt(x)$; or $\qpt(x) > q > \qmt(x)$. If the set of
%quantiles is a closed interval $[\mmin, \mmax]$, for any $x$ such that $f(x) =
%\mmax$ we have two cases to consider: $\qpt(x) = q = \qmt(x)$; or
%$\qpt(x) > q = \qmt(x)$. Note that if one of those holds for some such
%$x$, then it holds for any such $x$.}
\renewcommand{\mmax}{m}
Hereunder, we abbreviate $\mmax$ for the
$q$-quantile value $\QPtq(f)$ of $f$ under $P_{\theta^t}$.
%Hereunder, we let $\mmax = \max \{\QPtq(f)\}$ even for the case $\{\QPtq(f)\}$ is a singleton.

%Let us also abbreviate $q^\leq=P_{\theta^t}[f(x)\leq m]$ and
%$q^<=P_{\theta^t}[f(x)< m]$. We have $q^\leq \geq q\geq q^<$ by definition of
%$m$.

Let us compute the weighted preference $\Wft(x)$. Since the selection
scheme $w$ satisfies $0\leq w(u)\leq 1/q$ for all $u\in[0;1]$, we have
$0\leq \Wft(x)\leq 1/q$ for any $x$.

We claim that $f(x)>m$ implies $\Wft(x)=0$. Indeed, 
suppose that $x$ is such that $f(x)>m$. Since by definition $m$ is
the largest value such that $P_{\theta^t}[y:f(y)\geq m]\geq 1-q$, we must
have $P_{\theta^t}[y:f(y)\geq f(x)]<1-q$. Hence
$P_{\theta^t}[y:f(y)<f(x)]>q$, i.e.,
$\qmt(x)>q$. Now this implies $\Wft(x)=0$ for our choice of selection
scheme $w$.

Thus $\Wft(x)$ is at most $1/q$ and vanishes if $f(x)>m$. For any
probability distribution $P_\theta$, this implies that
\[
\Jt{\theta}=\E_{P_\theta} [\Wft(x)]\leq \frac1q\, P_{\theta}[x:f(x)\leq m]
\enspace.
\]
Therefore, 
\begin{align*}
 \Jt{\theta} > 1 
 \Longrightarrow{}& P_\theta[x: f(x) \leq \mmax] > q \\
 \Longrightarrow{}& \QPthetaq(f) \leq \mmax \enspace.
\end{align*}
If moreover $P_\theta[x: f(x) = \mmax] = 0$, we have $P_\theta[x: f(x)
\leq \mmax] = P_\theta[x: f(x) < \mmax]$ hence
\begin{align*}
 \Jt{\theta} > 1 
 \Longrightarrow{}& P_\theta[x: f(x) < \mmax] > q \\
 \Longleftrightarrow{}& P_\theta[x: f(x) \geq \mmax] < 1 - q \\
 \Longrightarrow{}& \QPthetaq(f) < \mmax \enspace.
\end{align*}

Altogether, $\Jt{\theta^{t+\deltat}} > 1$ implies quantile improvement
$\QPttq(f) \leq \QPtq(f)$. Moreover, if 
$P_{\theta^{t+\deltat}}[x: f(x) = \mmax] = 0$, we have strict quantile 
improvement $\QPttq(f) < \QPtq(f)$.
\end{proof}

This completes the proof of Theorem~\ref{thm:qimp}. 

%\yann{Now replaced with the Corollary above}
%\del{
%As we have mentioned, IGO-ML is equivalent to the IGO algorithm and
%the CE/ML method for an exponential family with the exponential
%parametrization. Therefore Theorem~\ref{thm:qimp} guarantees $q$-quantile improvement in the IGO algorithm and the CE/ML method for such exponential families on any search spaces.
%}

\begin{example}
Bernoulli distributions constitute an exponential family where the
sufficient statistics $T_i(x)$ are $x_i$. The parameter $\theta$ used in
PBIL (Example~\ref{ex:pbil}) is indeed the expectation parameter. Thus,
PBIL is an instance of IGO-ML and can be viewed as a CE/ML method at the
same time. Hence, by Theorem~\ref{thm:qimp}, each
infinite-population PBIL step leads to $q$-quantile improvement 
if we employ $q$-truncation selection, which is not the same as
the exponential weights introduced in \cite{Baluja1995icml}.
\end{example}

% \del{
% \begin{remark}
% We have not treated the case $\deltat=1$, which corresponds to a full
% maximum likelihood update. In this case we only get $\Jt{\theta} \geq 1$
% which is generally not enough to conclude that quantiles improve, in case
% the $q$-quantile value is not uniquely defined. Still, if the family of
% distributions $\{P_\theta\}$ has the property that distinct
% distributions $P_{\theta}$ and $P_{\theta'}$ never coincide on a subset of
% the search space with non-zero measure (a property satisfied, e.g., by
% Gaussian distributions), then Jensen's inequality~\eqref{eq:JandH} is
% strict, so that $\Jt{\theta} > 1$ and the conclusions of
% Theorem~\ref{thm:qimp} hold.
% \end{remark}
% }

\begin{remark}
The proof of the theorem is quantitative: the Kullback--Leibler divergence
$\KL{P_{\theta^t}}{P_{\theta^{t+\deltat}}}$ indicates how much progress
was made. More precisely (assuming for simplicity a continuous situation
with no plateaus),
while the probability under $P_{\theta^t}$ to fall into the best $q$ percent
of points for $P_{\theta^t}$ is $q$ by definition, the probability under
$P_{\thetatt}$ to fall into the best $q$ percent
of points for $P_{\theta^t}$ is at least
$q \exp \left(
\frac{1 - \deltat}{\deltat} \KL{P_{\theta^t}}{P_{\theta^{t+\deltat}}}
\right)$.
\end{remark}

\subsection{Blockwise IGO-ML}

The expectation parameter is not always the most obvious one. When it
comes to multivariate Gaussian distributions, the expectation parameter
is the mean vector and second moment, ($m$, $m m^\T + C$). Meanwhile,
the CMA-ES and the CE/ML method for continuous optimization parametrize
the mean vector and covariance matrix, hence they differ from the
IGO-ML algorithm. Moreover, sometimes different step sizes (learning
rates) are employed for each parameter, which makes the direction of
parameter update different from that of the natural gradient. Here, we justify
some of these settings by guaranteeing $q$-quantile improvement in an
extended framework.

We define an extension of IGO-ML, \emph{blockwise IGO-ML}, that
recovers the pure rank-$\mu$ CMA-ES update with different learning rates for $m$ and $C$.

\begin{definition}
Let $\theta = (\theta_1, \dots, \theta_k)$ be any decomposition of the
parameter $\theta$ into $k$ blocks, 
and let $\{\deltat_1, \dots, \deltat_k\}$ be a step size for each block. 
For $1\leq j\leq k$, define the \emph{$j$-th block partial IGO-ML update} with
step size $\deltat_j$ as the map sending a parameter value $\theta$ to
$\Phi_j(\theta)$ where
\begin{multline}
\label{eq:partial-IGOML}
\Phi_j(\theta)
\deq \argmax_{\begin{subarray}{c}
\theta^*\\
\theta^*_i=\theta_i\text{ for all }i\neq j
\end{subarray}
}
\Bigg\{ (1 - \deltat_j) \E_{P_\theta}\left[ \ln p_{\theta^*}(x) \right]   \\
+ \deltat_j \sum_i \wi \ln p_{\theta^*}(x_i) \Bigg\} \enspace.
\end{multline}
The {\em blockwise IGO-ML} updates the parameter $\theta$ as follows.
Given a current parameter value $\theta^t$, update the first block of
$\theta^t$, then the second block, etc., in that order; explicitly, set
\begin{equation}
 \theta^{t+1}\deq 
 (\Phi_k\circ\cdots\circ
 \Phi_2\circ\Phi_1)(\theta^t)
 \enspace,
\label{eq:block-IGOML}
\end{equation}
where we note that the same Monte Carlo sample $\{x_i\}$ from $P_{\theta^t}$ is used
throughout the whole range of block updates $\Phi_1,\ldots,\Phi_k$.
\end{definition}

The infinite-population step ($\lambda = \infty$) reads the same with
\begin{multline}
\Phi_j(\theta)
\deq \argmax_{\begin{subarray}{c}
\theta^*\\
\theta^*_i=\theta_i\text{ for all }i\neq j
\end{subarray}
}
\Bigg\{ (1 - \deltat_j) \E_{P_\theta}\left[ \ln p_{\theta^*}(x) \right]   \\
+ \deltat_j \,\E_{P_{\theta^t}} \bigl[\Wft(x)  \ln p_{\theta^*}(x) \bigr]
\Bigg\} \enspace.
\label{eq:ideal-block-IGOML}
\end{multline}

As before, the finite- and infinite-population blockwise IGO-ML
updates 
only make
sense if the $\argmax$ in \eqref{eq:partial-IGOML} or 
\eqref{eq:ideal-block-IGOML} is uniquely determined.

Note that the blockwise IGO-ML depends on the decomposition of the
parameters into blocks and their update order, while it is independent of
the parametrization inside each block. Blockwise IGO-ML is not
necessarily equivalent to IGO-ML even when all $\deltat_i$ are equal
to $\deltat$.

\yohe{I was looking for a property of parametrizations with which the IGO algorithm with different learning rates turns into the block IGO-ML, but I haven't found it yet... }

\begin{proposition}
\label{prop:igo-cma}
The pure rank-$\mu$ CMA-ES update (Example~\ref{ex:cmaes}) is an instance
of blockwise IGO-ML for Gaussian distributions, with parameter
decomposition $\theta = (\theta_1, \theta_2)$ where $\theta_1 = C$, the
covariance matrix, and
$\theta_2 = m$, the mean vector.
\end{proposition}

\begin{proof}
Given $\theta^t = (C^t, m^t)$, blockwise IGO-ML first updates $C$ as
follows:
\begin{multline}
 C^* = \argmax_{C} \Bigg\{ (1 - \deltat_C) \E_{P_{(C^t, m^t)}}\left[ \ln p_{(C, m^t)}(x) \right]   \\
+ \deltat_C \sum_i \wi \ln p_{(C, m^t)}(x_i) \Bigg\} \enspace.
\label{eq:c-igoml}
\end{multline}
Considering $\{ P_{(C, m^t)} \}$ as an exponential family of Gaussian distributions
whose mean vector is fixed to $m^t$, \eqref{eq:c-igoml} can be viewed as
an ordinary IGO-ML step for this restricted model. Then, since
(after shifting the origin of the coordinate system to $m^t$) 
$C$ is the expectation parameter of
the restricted model, the update is given by~\eqref{eq:igo-equiv-sample}
namely
\begin{equation*}
 C^* = C^t + \deltat_C \sum_{i=1}^{\lambda} \wi \bigl( (x_i - m^t)(x_i - m^t)^\T - C^t \bigr) \enspace.
\end{equation*}
Next, $m$ is updated as
\begin{multline*}
 m^* = \argmax_{m} \Bigg\{ (1 - \deltat_m) \E_{P_{(C^*, m^t)}}\left[ \ln p_{(C^*, m)}(x) \right]   \\
+ \deltat_m \sum_i \wi \ln p_{(C^*, m)}(x_i) \Bigg\} \enspace.
\end{multline*}
To derive $m^*$, let us differentiate the inside of $\argmax$
w.r.t.\ $m$ and derive the zero point of the derivative. Seeing that
$\nabla_{\!m} \ln p_{(C^*, m)}(x) = (C^*)^{-1} ( x - m )$, we find the condition
\begin{multline*}
 (1 - \deltat_m) (C^*)^{-1}( m^t - m^*) + \deltat_m \sum_i \wi
 (C^*)^{-1}(x_i -  m^*) = 0 \enspace,
\end{multline*}
which holds if and only if
\begin{equation*}
 m^* = m^t + \deltat_m \sum_{i=1}^{\lambda} \wi (x_i - m^t) \enspace.
\end{equation*}
This is equivalent to the pure rank-$\mu$ CMA-ES update.
\end{proof}

\yohe{%
If the order of the parameters is changed, namely $\theta_1 = m$ and $\theta_2 = C$, the blockwise IGO-ML recovers EMNA, which is known to cause premature convergence even on Sphere function.
}%

Quantile improvement as in Theorem~\ref{thm:qimp} readily extends to
this setting as follows.

\begin{theorem}
\label{thm:iter-qimp}
Let the selection scheme be $w(u) = \indicator{u \leq q} / q$ where $0 <
q < 1$. Assume that the $\argmax$ defining each partial 
infinite-population IGO-ML update \eqref{eq:ideal-block-IGOML} 
is uniquely determined. Then for $0
< \deltat_j \leq 1$ ($ j \in \llbracket1; k\rrbracket$), each
infinite-population blockwise IGO-ML step \eqref{eq:block-IGOML} leads to
$q$-quantile improvement: $Q_{P_{\theta^{t+1}}}^q(f) \leq \QPtq(f)$.

Moreover, equality can hold only if
$P_{\theta^{t+1}}=P_{\theta^t}$ or if $P_{\theta^{t+1}}[x: f(x) =
\QPtq(f)] > 0$.
\end{theorem}

Consequently, each infinite-population step of the pure rank-$\mu$
CMA-ES update guarantees $q$-quantile
improvement. Indeed, from Proposition~\ref{prop:igo-cma} this variant of
the CMA-ES is an instance of blockwise IGO-ML.
%\del{
%Theorem~\ref{thm:iter-qimp} guarantees $q$-quantile
%improvement for the CMA-ES. Indeed, the pure rank-$\mu$ CMA-ES
%update
%is an instance of the blockwise IGO-ML as we have proved in
%Proposition~\ref{prop:igo-cma}, implying that the \new{infinite-population} step of this
%variant of CMA-ES leads to $q$-quantile improvement unless
%$P_{\theta^{t+1}} = P_{\theta^t}$, i.e.,  $\theta^{t+1} = \theta^{t}$.}
Moreover, if each level set of $f$ has zero Lebesgue measure, which often holds for continuous optimization, we have strict $q$-quantile improvement.

%\end{example}

%<<<<<<< .mine
%\begin{proof}
%According to Lemma~\ref{lem:JandQ}, it is enough both for $q$-quantile improvement and for strict $q$-quantile improvement to show that $\Jt{\theta^{t+1}} > 1$. Because of the assumptions on $P_\theta$, inequality \eqref{eq:JandH} holds with equality if and only if $P_{\theta} = P_{\theta^t}$. Therefore, it is enough to show $\Ht(\theta^{t + \deltat}) \geq \Ht(\theta^t)$.
%
%=======
\newcommand{\Qqnext}{Q^q_{P_{\theta^{t+1}}}(f)}
\begin{proof}
If $P_{\theta^{t+1}} = P_{\theta^{t}}$, obviously
$Q^q_{P_{\theta^{t+1}}}(f) = \QPtq(f)$. We assume 
$P_{\theta^{t+1}} \neq P_{\theta^{t}}$ in the following.

Set $\theta^{t,0}\deq \theta^t$ and $\theta^{t,j}\deq
\Phi_j(\theta^{t,j-1})$ so that $\theta^{t+1}=\theta^{t,k}$.
According to Lemma~\ref{lem:JandQ}, to prove quantile improvement it is
enough to show that $\Jtwo{\theta^{t+1}}{\theta^{t}}>1$. Moreover,
this implies strict quantile improvement provided $P_{\theta^{t+1}}[x: f(x) 
= \QPtq(f)] = 0$.

According to \eqref{eq:JandH}, if $\Ht(\theta^{t+1}) > \Ht(\theta^t)$ we
have \new{that} $\Jt{\theta^{t+1}} > 1$. To show that $\Ht(\theta^{t+1}) >
\Ht(\theta^t)$ we decompose $\Ht(\theta^{t+1}) - \Ht(\theta^t)$
into the sum of partial differences, namely,
\begin{equation*}
\Ht(\theta^{t+1}) - \Ht(\theta^t) = \sum_{j=1}^{k} \Ht(\theta^{t,j}) - \Ht(\theta^{t,j-1}) \enspace, 
\end{equation*}
and we will prove that each term is non-negative. Moreover, if
$P_{\theta^{t,j}} \neq P_{\theta^{t,j-1}}$ for some
$j\in\llbracket 1;k\rrbracket$, we will have $\Ht(\theta^{t,j}) -
\Ht(\theta^{t,j-1}) > 0$ for this $j$.
%\del{, which implies $\Ht(\theta^{t+1}) - \Ht(\theta^t) > 0$} 
Since $P_{\theta^{t+1}}
\neq P_{\theta^{t}}$ implies $P_{\theta^{t,j}} \neq P_{\theta^{t,j-1}}$ for
at least one $j\in\llbracket 1;k\rrbracket$,
we will have that $\Ht(\theta^{t+1}) - \Ht(\theta^{t}) > 0$, resulting in $\Jt{\theta^{t+1}} > 1$.

We proceed as in Theorem~\ref{thm:qimp}. Since
$\theta^{t,j}=\Phi_j(\theta^{t,j-1})$ is the only maximizer of \eqref{eq:ideal-block-IGOML}, we have
\begin{multline*}
  (1 - \deltat_j) \E_{P_{\theta^{t,j-1}}}\left[ \ln p_{\theta^{t,j}}(x) \right] 
  + \deltat_j \Ht(\theta^{t,j}) \\
 \geq (1 - \deltat_j) \E_{P_{\theta^{t,j-1}}}\left[ \ln p_{\theta^{t,j-1}}(x) \right] 
  + \deltat_j \Ht(\theta^{t,j-1}) 
\end{multline*}
with equality holding if and only if $\theta^{t,j} = \theta^{t,j-1}$.
Rearranging, we get
\begin{equation*}
\Ht(\theta^{t,j})-\Ht(\theta^{t,j-1})
> \frac{1 - \deltat_j}{ \deltat_j }
\KL{P_{\theta^{t,j-1}}}{P_{\theta^{t,j}}}
\end{equation*}
if $\theta^{t,j} \neq \theta^{t,j-1}$, and $\Ht(\theta^{t,j}) =
\Ht(\theta^{t,j-1})$ if $\theta^{t,j} = \theta^{t,j-1}$. The
right-hand side of the above inequality is non-negative for $0 < \deltat
\leq 1$. Therefore, $\Ht(\theta^{t,j})-\Ht(\theta^{t,j-1}) \geq 0$ for
all $j \in \llbracket 1; k\rrbracket$.
Moreover, since $P_{\theta^{t+1}} \neq P_{\theta^{t}}$, for at least one
$j \in \llbracket 1;
k\rrbracket$ we have $\theta^{t,j} \neq \theta^{t,j-1}$ and thus
$\Ht(\theta^{t,j})-\Ht(\theta^{t,j-1}) > 0$ for
this $j$, implying that $\Ht(\theta^{t+1})-\Ht(\theta^{t}) > 0$. 
This completes the proof.
\end{proof}

%\yohe{I have been editing the following, but decided not to include it.
% The proof is incomplete.}
%\yann{I'm rather for keeping it; I think your new proof is complete...
%and shorter than the statement!}

\subsection{Finite Population Sizes}

The results above are valid for ``ideal'' updates
with infinite sample size. With finite sample size, the update
\eqref{eq:IGOML} defines a stochastic sequence (depending on the random sample
$\{x_i\}$) and so one cannot expect monotone $q$-quantile improvement at
each step. Still, we can expect $q$-quantile improvement with high
probability when the population size is sufficiently large.

We provide an analogue of Theorem~\ref{thm:qimp} for finite but large
population size. A similar statement holds for blockwise IGO-ML.
The proof follows a standard probabilistic approximation
argument.

\begin{proposition}
Let $w(\cdot)$ be the $q$-truncation selection scheme: $w(u) = \indicator{u \leq q} / q$ where $0 <
q < 1$. Let $\{P_\theta\}$ be an exponential family of probability
distributions, parametrized by its expectation parameter.
Assume that the $\argmax$ defining the infinite-population IGO-ML step
\eqref{eq:ideal-IGOML} is uniquely defined.

Assume that for all $\theta\in \Theta$, the derivative $\partial \ln
P_\theta(x)/\partial \theta$ exists for $P_\theta$-almost all $x\in X$
and has finite second moment: $\E_{P_\theta} \left[\left|\partial \ln
P_\theta(x)/\partial \theta\right|^2\right]<\infty$.

Let $0<\deltat\leq 1$.
Let $\theta^{t+\deltat}_\lambda$ be the IGO-ML update \eqref{eq:IGOML}
with sample size $\lambda$, and let $\theta^{t+\deltat}_\infty$ be the
infinite-population IGO-ML update \eqref{eq:ideal-IGOML}.
Assume that $\theta^{t+\deltat}_\infty\neq \theta^t$.

Then, with probability tending to $1$ as $\lambda\to\infty$, the
finite-population update $\theta^{t+\deltat}_\lambda$ results in
$q$-quantile improvement:
\[
Q_{P_{\theta^{t+\deltat}_\lambda}}^q(f)
\leq
Q_{P_{\theta^t}}^q(f)
\enspace.
\]
\end{proposition}

Consequently, the same holds for the CE/ML method and the IGO
algorithm when they are applied to an exponential family using the
expectation parameters.

Note the assumption that the \emph{ideal} dynamics has not reached
equilibrium yet: $\theta^{t+\deltat}_\infty\neq \theta^t$. If
$\theta^{t+\deltat}_\infty=\theta^t$, the finite-population dynamics will
just randomly wander around this equilibrium value with some noise,
resulting in either improvement or deterioration at each step.

Also note that the population size $\lambda$ needed may depend on the
current location $\theta^t$ in parameter space, as well as the objective
function $f$. For instance, highly oscillating functions $f$ likely
require higher population sizes for a consistent estimation of the IGO-ML
update.

%\begin{theorem}
%Let the selection scheme be $w(u) = \indicator{u \leq q} / q$ where $0 <
%q < 1$. Assume that the variances of all sufficient statistics are finite, namely
%$\E_{P_{\theta^t}}[ \norm{T_i(x) - \ep_i^t}^2 ] < \infty$
%for all $i \in \llbracket 1; n\rrbracket$.
%Then for $0 < \deltat < 1$, each IGO step \eqref{eq:igo-exp-sample} leads to
%$q$-quantile improvement: $\QPttq(f) \leq \QPtq(f)$ with high probability.
%\end{theorem}

\begin{proof}
For exponential families, the IGO and IGO-ML updates coincide.
Under the conditions of the theorem, the finite-population IGO update \eqref{eq:Gt} is a consistent 
estimator of the infinite-population IGO update \eqref{eq:ng} 
\cite[Proposition~18]{Arnold2011arxiv}, implying that
$\thetatt_\lambda$
converges with probability one to $\thetatt_\infty$.
Under our regularity assumptions on $P_\theta$, this implies pointwise convergence of 
$p_{\thetatt_\lambda}$ to $p_{\thetatt_\infty}$, which, since $\Wft(x)$
is bounded, leads to
\begin{multline*}
\Jt{\thetatt_\lambda} = \E_{P_{\thetatt_\lambda}}[ \Wft(x) ] \\
\to 
\E_{P_{\thetatt_\infty}}[ \Wft(x) ] = \Jt{\thetatt_\infty}
\quad \text{as }\lambda \to \infty.
\end{multline*}
Now the right-hand side is greater than $1$ for $0 < \deltat \leq 1$ unless 
$\theta^{t+\deltat}_\infty = \theta^t$, as we have shown in the proof of 
Theorem~\ref{thm:qimp}. Thus, we have $\Jt{\thetatt_\lambda} > 1$ with high 
probability for sufficiently large $\lambda$. Thus
Lemma~\ref{lem:JandQ} entails $q$-quantile improvement with high 
probability.
\end{proof}

\section{Fitness-Proportional Selection}
\label{sec:fit}

These results carry over to the use of a composite $g \circ f$ of a
%\del{non-increasing}\yann{I do not think we use this} 
function $g$ with the objective function $f$, as a
selection weight instead of
$\Wft$ in the IGO framework. This covers, for instance, fitness-proportional
selection ($g=\mathrm{Id}$).
%\yohe{Not sure if it is a good idea to include the following.}\yann{OK,
%let's remove it.}
%
%\del{In this case the role of $\Jt{\theta}$ is played by the expectation
%$\E_{P_\theta}[g \circ f(x)]$ of $g \circ f(x)$ over $P_\theta$. Then, it
%is easy to see that an infinitesimal step along the natural gradient of
%$\E_{P_\theta}[g \circ f(x)]$ leads to monotone strict improvement of the
%expected value of the composite function.}
We prove that, when considering the
natural gradient ascent for an exponential family \eqref{eq:exponential}
using the expectation parameter \eqref{eq:ep}, we can guarantee monotone
$\E_{P_\theta}[g \circ f(x)]$-value improvement for 
updates of step size inversely proportional to $\E_{P_\theta}[g \circ f(x)]$. More precisely,
%\newpage% To make the theorem not to spread in two pages
\begin{theorem}
\label{thm:fitness-imp}
Assume $g \circ f$ is non-negative and not \nnew{almost} everywhere $0$. Consider
the update
\begin{equation}
 \theta^{t+\deltat} = \theta^t + \deltat\, \E_{P_{\theta^t}}\left[ \frac{g \circ f(x)}{\E_{P_{\theta^t}}[g \circ f(x)]} \left( T(x) - \theta^t \right) \right] \enspace, 
\label{eq:fit-igo}
\end{equation}
where $\theta = \ep$ is the expectation parameter of the exponential
family $\{P_\theta\}$.

Then for $0 < \deltat \leq 1$, we have
\[
\E_{P_{\theta^{t+\deltat}}}[g \circ f(x)] \geq \E_{P_{\theta^t}}[g \circ f(x)]
\enspace .
\]
Moreover, equality
can occur
only if $P_{\theta^{t+\deltat}} = P_{\theta^{t}}$.
\end{theorem}

%<<<<<<< .mine
%\new{Gradient based methods with fitness-proportional selection are often employed especially in reinforcement learning, e.g.\ \emph{policy gradient with parameter-based exploration} (PGPE) \cite{Sehnke2010nn}. One disadvantage of gradient based method is considered that the step size have to be calibrated by users depending on problems, and alternative methods such as \emph{expectation-maximization} based method including RPP \cite{Dayan1997nc} are sometimes employed to avoid this issue. Theorem~\ref{thm:fitness-imp}, however, ensures that each natural gradient step improves the expected fitness for $0 < \deltat \leq 1$ when an exponential family is used with its expectation parameters.}
%=======
Gradient based methods with fitness-proportional selection are often
employed, especially in reinforcement learning, e.g.\, \emph{policy
gradient with parameter based exploration} (PGPE) \cite{Sehnke2010nn}.
One disadvantage of gradient based methods is that the step
size has to be calibrated by the user depending on the problem at hand. 
Alternative methods such as \emph{expectation-maximization} 
\cite{Dayan1997nc}, including the RPP below, are sometimes employed to avoid this issue.
Theorem~\ref{thm:fitness-imp}, however, ensures that each natural
gradient step improves the expected fitness for $0 < \deltat \leq 1$ when
an exponential family is used with its expectation parameters.

%\yann{Moved the example before the proof}

\begin{example}
The Relative Payoff Procedure (RPP) \cite{Hinton1989ai} is a reinforcement
learning algorithm, also known as expectation-maximization (EM) algorithm
for reinforcement learning \cite{Dayan1997nc}. The RPP expresses a policy on the action space $\X = \{0, 1\}^d$ by a Bernoulli distribution $P_\theta(x)$ parametrized by the expectation parameter. The objective to be maximized is the expectation $\E_{P_\theta}[r(x)]$ of non-negative reward $r(x)$ after taking action $x \in \X$. The RPP updates the parameters to
\begin{equation*}
 \theta^{t+1} = \frac{\E_{P_{\theta^t}}[x r(x) ]}{\E_{P_{\theta^t}}[r(x)]} \enspace.
\end{equation*}
Remember the sufficient statistics $T(x)$ for Bernoulli distributions are
$T_i(x) = x_i$. Thus the RPP is equivalent to \eqref{eq:fit-igo} with $g \circ
f(x) = r(x)$ and $\deltat = 1$ and can be viewed as a natural gradient
ascent with large step.

The RPP is known from \cite{Dayan1997nc} to monotonically improve
expected reward, thanks to its
expectation-maximization interpretation.
Theorem~\ref{thm:fitness-imp} can be thought of as an extension of this
result, and also shows monotone 
improvement for the smoothed RPP, where a step size $0 < \deltat \leq 1$ is
introduced.
\end{example}

\begin{proof}
Most of the proof of Theorem~\ref{thm:qimp} carries over. Replacing
$\Wft$ in \eqref{eq:igo-equiv} with $g \circ f / \E_{P_{\theta^t}}[g \circ f(x)]$, \eqref{eq:igo-equiv} still holds and we have
\begin{equation}
\begin{split}
  \theta^{t+\deltat} 
  &= \theta^t + \deltat \,\E_{P_{\theta^t}} \left[\frac{g \circ f(x)}{\E_{P_{\theta^t}}[g \circ f(x)]} \left( T(x) - \theta^t \right) \right] \\
  &= \argmax_\theta \bigg\{ (1 - \deltat) \E_{P_{\theta^t}} \left[ \ln p_\theta(x)\right] \\
 &\quad \qquad + \deltat \E_{P_{\theta^t}} \left[ \frac{g \circ f(x)}{\E_{P_{\theta^t}}[g \circ f(x)]} \ln p_\theta(x)\right] \bigg\}
\end{split}
\label{eq:f-igo-equiv}
\end{equation}

Thanks to Jensen's inequality, we have the counterpart of \eqref{eq:JandH} as
\begin{multline}
 \ln \E_{P_\theta}[g \circ f(x)] - \ln \E_{P_{\theta^t}}[g \circ f(x)] \\
 \geq \frac{\E_{P_{\theta^t}} \left[ g \circ f(x) \ln p_\theta(x)\right]}
 {\E_{P_{\theta^t}}[g \circ f(x)]}  
 -  \frac{\E_{P_{\theta^t}} \left[ g \circ f(x) \ln p_{\theta^t}(x)\right]}{\E_{P_{\theta^t}}[g \circ f(x)]} \enspace.
\label{eq:f-JandH}
\end{multline}
%\del{Moreover, since $\ln$ is a strictly concave function, the inequality is
%strict unless $P_\theta$ and $P_{\theta^t}$ coincide over the set where
%$g \circ f(x)>0$.\yann{now useless}}

Because of the second equality of \eqref{eq:f-igo-equiv}, we have the counterpart of \eqref{eq:HandKL} as
\begin{multline*}
 \frac{\E_{P_{\theta^t}} \left[ g \circ f(x) \ln p_{\theta^{t+\deltat}}(x)\right]}
 {\E_{P_{\theta^t}}[g \circ f(x)]}  
 -  \frac{\E_{P_{\theta^t}} \left[ g \circ f(x) \ln p_{\theta^t}(x)\right]}{\E_{P_{\theta^t}}[g \circ f(x)]} \\
\geq 
\frac{1 - \deltat}{\deltat} \KL{P_{\theta^t}}{P_{\theta^{t+\deltat}}},
\end{multline*}
and moreover, since the maximizer in \eqref{eq:f-igo-equiv} is
unique, the inequality is strict unless $\theta^t=\theta^{t+\deltat}$.
Hence, since the right-hand side is
non-negative, by \eqref{eq:f-JandH} we have $\ln \E_{P_\theta}[g \circ f(x)] \geq \ln \E_{P_{\theta^t}}[g \circ f(x)]$ with equality only if $P_{\theta^t} = P_{\theta^{t+\deltat}}$. This completes the proof.
\end{proof}

\begin{remark}
As mentioned in Remark~\ref{rem:malago}, Malag{\`o}~et~al.\
\cite{Malago2011foga} propose the natural gradient algorithm for discrete
optimization using exponential distributions. However, as they parametrize
the exponential distributions by the natural parameters $\theta = \np$,
%\del{which is out of the scope of Theorem~\ref{thm:fitness-imp}.
%Therefore,}
Theorem~\ref{thm:fitness-imp} does not guarantee expected fitness
improvement for their algorithm\new{s}, whereas it does so for
the algorithm \eqref{eq:malago-exp} using the expectation parameters.
\end{remark}

\section{Further Discussion}
\label{sec:disc}

\yann{Rename this to "conclusion"? I like "Further discussion" better,
but apparently one of the referees interprets it as "Further research"...}
\yohe{I feel ``Further discussion'' is better than ``Conclusion'' even though one of the referees misunderstood the section name. Because this section is really further discussion.}

These results contribute to bringing theory closer to practice, by
waiving the need for infinitesimal step sizes in gradient ascent. Still,
they cover only the ``ideal'' 
situation with infinite population size, as well as finite but very large
population sizes (by a standard probabilistic approximation argument).
Finite population sizes lead to stochastic behavior and so
monotone objective improvement at each step occurs only with high
probability.

In practice, population sizes used can be quite small, $\lambda\leq 10$,
with medium to small step sizes \cite{Hansen2003ec,Baluja1995icml}.
%\yohe{Should cite stochastic gradient algorithms?}\yann{Not sure.}  
It has been shown in 
\cite[Remark~2]{Akimoto2012ppsn}
that when population size does not tend to infinity, 
the expectation of the natural gradient estimate \eqref{eq:Gt}
 is the natural gradient \eqref{eq:ng} with a \emph{different} selection 
scheme $w$. 
So using the truncation weight
$w(u)=\indicator{u\geq q}$ with a small population size and very small
step sizes will 
result, by
the machinery of stochastic approximation
\cite{Kushner2003book,Borkar2008book},
in simulating an infinite-population
IGO step with another selection scheme, a situation outside the scope of
this article.  Our results, on the contrary, suggest using larger
populations and larger step sizes instead.

% \yann{Is there a selection scheme that, used with a small population,
% would result in simulating the infinite-population with truncation
% selection? I fear there is too much ``smoothing'' in the Bernoulli
% combination of weights for this to be possible.}
% 
% \yohe{I am afraid that I have to say it is impossible to reach $q$-truncation selection with finite populations. At this point, we can not relate the IGO with a small population to the IGO with the infinite population with truncation selection.}

Finally, let us stress that objective improvement is not, by itself, a
sufficient guarantee that optimization performs well: in situations of
premature convergence, the objective still improves at each step. Premature
convergence can occur for large values of the learning rate in some
instantiations of IGO and IGO-ML (see the study in \cite{Arnold2011arxiv}); our
results say nothing about this phenomenon.

%\newpage 

% \todo{Do we need a related work section? I'm not a big fan of the
% mandatory "Related work" and "Conclusion"...}
% 
% \yohe{My opinion on ``related work'' is that we should mention it only when it will help to understand the motivation and importance of the paper. I don't think we need to have a related work section in this paper. 
% 
% Let's add a related work section if we are required by reviewers.}

%ACKNOWLEDGMENTS are optional
%\section{Acknowledgments}

%\bibliographystyle{abbrv}
%\bibliography{lib}  % sigproc.bib is the name of the Bibliography in this case

\begin{thebibliography}{10}

\bibitem{Akimoto2012ppsn}
Y.~Akimoto, A.~Auger, and N.~Hansen.
\newblock Convergence of the continuous time trajectories of isotropic
  evolution strategies on monotonic $\mathcal{C}^2$-composite functions.
\newblock In {\em Parallel Problem Solving from Nature - PPSN XII, 12th
  International Conference}, number 7491 in Lecture Notes in Computer Science,
  pages 42--51, Taormina, Italy, September 1--5 2012. Springer.

\bibitem{Akimoto2010ppsn}
Y.~Akimoto, Y.~Nagata, I.~Ono, and S.~Kobayashi.
\newblock Bidirectional relation between {CMA} evolution strategies and natural
  evolution strategies.
\newblock In R.~Schaefer, C.~Cotta, J.~Kolodziej, and G.~Rudolph, editors, {\em
  Parallel Problem Solving from Nature - PPSN XI, 11th International
  Conference}, volume 6238 of {\em Lecture Notes in Computer Science}, pages
  154--163, Krak{\'o}w, Poland, September 11--15 2010. Springer.

\bibitem{Amari1998nc}
S.-i. Amari.
\newblock Natural gradient works efficiently in learning.
\newblock {\em Neural Computation}, 10(2):251--276, 1998.

\vfill\eject

\bibitem{Amari2000book}
S.-i. Amari and H.~Nagaoka.
\newblock {\em Methods of Information Geometry}.
\newblock Translations of Mathematical Monographs vol.~191. American
  Mathematical Society, 2000.

\bibitem{Arnold2011arxiv}
L.~Arnold, A.~Auger, N.~Hansen, and Y.~Ollivier.
\newblock {I}nformation-{G}eometric {O}ptimization algorithms: A unifying
  picture via invariance principles.
\newblock {\em arXiv:1106.3708v1}, 2011.

\bibitem{Baluja1995icml}
S.~Baluja and R.~Caruana.
\newblock Removing the genetics from the standard genetic algorithm.
\newblock In {\em Proceedings of the 12th International Conference on Machine
  Learning}, pages 38--46, 1995.

\bibitem{Boer2005aor}
P.-T.~D. Boer, D.~P. Kroese, S.~Mannor, and R.~Y. Rubinstein.
\newblock A tutorial on the cross-entropy method.
\newblock {\em Annals of Operations Research}, (134):19--67, 2005.

\bibitem{Borkar2008book}
V.~S. Borkar.
\newblock Stochastic approximation: a dynamical systems viewpoint.
\newblock Cambridge University Press, 2008.

\bibitem{Dayan1997nc}
P.~Dayan and G.~E. Hinton.
\newblock Using expectation-maximization for reinforcement learning.
\newblock {\em Neural Computation}, 9(2):271--278, 1997.

\bibitem{Dorigo1996tsmc}
M.~Dorigo, V.~Maniezzo, and A.~Colorni.
\newblock The ant system: Optimization by a colony of cooperating agents.
\newblock {\em IEEE Transactions on Systems, Man, and Cybernetics - part B},
  26(1):1--13, 1996.

\bibitem{Hansen2003ec}
N.~Hansen, S.~D. Muller, and P.~Koumoutsakos.
\newblock Reducing the time complexity of the derandomized evolution strategy
  with covariance matrix adaptation ({CMA-ES}).
\newblock {\em Evolutionary Computation}, 11(1):1--18, 2003.

\bibitem{Hinton1989ai}
G.~E. Hinton.
\newblock Connectionist learning procedures.
\newblock {\em Artificial Intelligence}, 40(1-3):185--234, 1989.

\bibitem{Kullback}
S.~Kullback.
\newblock {\em Information theory and statistics}.
\newblock Dover Publications Inc., Mineola, NY, 1997.
\newblock Reprint of the second (1968) edition.

\bibitem{Kushner2003book}
H.~J. Kushner and G.~G. Yin.
\newblock {\em Stochastic approximation and recursive algorithms and
  applications}.
\newblock Springer Verlag, 2nd edition, 2003.

\bibitem{Larranaga2002book}
P.~Larra{\~n}aga and J.~A. Lozano.
\newblock {\em Estimation of Distribution Algorithms: A New Tool for
  Evolutionary Computation}.
\newblock Estimation of Distribution Algorithms: A New Tool for Evolutionary
  Computation. Kluwer Academic Publishers, 2002.

\bibitem{Malago2011foga}
L.~Malag{\`o}, M.~Matteucci, and G.~Pistone.
\newblock Towards the geometry of estimation of distribution algorithms based
  on the exponential family.
\newblock In H.-G. Beyer and W.~B. Langdon, editors, {\em FOGA '11: Proceedings
  of the 11th workshop proceedings on Foundations of genetic algorithms}, pages
  230--242. ACM, 2011.

\bibitem{Sehnke2010nn}
F.~Sehnke, C.~Osendorfer, T.~R{\"u}ckstie{\ss}, A.~Graves, J.~Peters, and
  J.~Schmidhuber.
\newblock Parameter-exploring policy gradients.
\newblock {\em Neural Networks}, 23(4):551--559, 2010.

\end{thebibliography}
% ACM needs 'a single self-contained file'!

%\balancecolumns
%\appendix

\end{document}